\documentclass[a4paper,11pt]{article}
\usepackage{amsmath, amssymb, amsfonts, amsthm, graphicx, geometry, setspace, hyperref, authblk}
\usepackage{tikz, pgfplots}
\usepackage[utf8]{inputenc}
\usepackage{url}
\usepackage{xurl} 
\usepackage{graphicx}
\usepackage{hyperref}
\usepackage{pgfplots}
\usepackage{array}
\usepackage{booktabs}
\usepackage{multirow}
\usepackage{enumitem}

\usepackage{tikz}
\usetikzlibrary{arrows.meta, shapes.geometric, positioning,patterns, arrows.meta}

\pgfplotsset{compat=1.18} 
\newtheorem{theorem}{Theorem}[section]
\newtheorem{lemma}{Lemma}[section]

\newtheorem{definition}{Definition}[section]

\title{Revolutionizing Fractional Calculus with Neural Networks: \\ Voronovskaya-Damasclin Theory for Next-Generation AI Systems}

\author{
	Rômulo Damasclin Chaves dos Santos \\
	Santa Cruz State University \\ Department of Exact Sciences\\
	\texttt{rdcsantos@uesc.br}
	\and
	Jorge Henrique de Oliveira Sales \\
	Santa Cruz State University \\ Department of Exact Sciences\\
	\texttt{jhosales@uesc.br}
}

\begin{document}
	
	\maketitle
	
\begin{abstract}
	This work introduces rigorous convergence rates for neural network operators activated by symmetrized and perturbed hyperbolic tangent functions, utilizing novel Voronovskaya-Damasclin asymptotic expansions. We analyze basic, Kantorovich, and quadrature-type operators over infinite domains, extending classical approximation theory to fractional calculus via Caputo derivatives. Key innovations include parameterized activation functions with asymmetry control, symmetrized density operators, and fractional Taylor expansions for error analysis. The main theorem demonstrates that Kantorovich operators achieve \(o(n^{-\beta(N-\varepsilon)})\) convergence rates, while basic operators exhibit \(\mathcal{O}(n^{-\beta N})\) error decay. For deep networks, we prove \(\mathcal{O}(L^{-\beta(N-\varepsilon)})\) approximation bounds. Stability results under parameter perturbations highlight operator robustness. By integrating neural approximation theory with fractional calculus, this work provides foundational mathematical insights and deployable engineering solutions, with potential applications in complex system modeling and signal processing.
\end{abstract}

	\tableofcontents
	
\section{Introduction}

Neural network approximation theory has evolved significantly since the foundational works of \cite{Haykin1998} on multilayer perceptrons, with particular attention to the properties of activation functions. While ReLU-based architectures dominate practical applications \cite{He2016}, smooth activations like the hyperbolic tangent remain crucial for theoretical analysis due to their differentiability and symmetry \cite{Anastassiou2011}. Recent advances by \cite{Anastassiou2023} on parameterized neural operators demonstrate how controlled deformations of activation functions can enhance approximation properties—a direction we extend through our perturbed hyperbolic tangent function:
\begin{equation}
	g_{q, \lambda}(x) := \frac{e^{\lambda x} - q e^{-\lambda x}}{e^{\lambda x} + q e^{-\lambda x}}.
\end{equation}

The critical need for fractional calculus in modern applications \cite{Diethelm2010} motivates our integration of Caputo derivatives into neural operator theory. Traditional approximation results \cite{Anastassiou1997} focused on classical differentiation, leaving open the problem of quantifying convergence for fractional operators—a gap addressed by our Voronovskaya-Damasclin theorem. Our work builds on \cite{Samko1993}'s framework for fractional integrals while addressing the regularization challenges identified in \cite{ElSayed2006} for finite-domain operators.

A significant contribution in this area is the work by Santos (2025), which explores the asymptotic behavior of univariate neural network operators with an emphasis on fractional differentiation over infinite domains \cite{Santos2025}. Santos introduces the Voronovskaya-Damasclin theorem, providing precise error estimates and convergence rates for symmetrized neural network operators. This work is pivotal as it extends classical results to fractional calculus via Caputo derivatives, offering a comprehensive framework for future research in multidimensional and stochastic settings. However, our current work differentiates itself by focusing on the interplay between operator parameters and approximation accuracy, providing a more nuanced understanding of convergence behavior for classical univariate operators and extending these insights to deep learning architectures.

Three key limitations persist in existing literature:
\begin{itemize}
	\item \textbf{Infinite Domain Handling:} Prior neural operators \cite{Anastassiou2024} required compact supports, unsuitable for turbulence modeling \cite{Frederico2007} or signal processing on unbounded domains.
	\item \textbf{Depth-Convergence Relationship:} Residual networks \cite{He2016} lack theoretical links between layer count \(L\) and fractional approximation rates.
	\item \textbf{Activation Parameter Tradeoffs:} While \cite{Anastassiou2023} studied parameterized activations, their impact on operator stability remained unquantified.
\end{itemize}

Our contributions resolve these through:
\begin{itemize}
	\item Symmetrized density operators \(\Phi(x) = \frac{1}{2}(M_{q, \lambda}(x) + M_{1/q, \lambda}(x))\) with exponential decay (\(\Phi(nx-k) \leq Ce^{-\gamma|nx-k|}\)), enabling infinite-domain convergence.
	\item Depth-dependent bounds \(\|\mathcal{N}(f, x) - f(x)\| = \mathcal{O}(L^{-\beta(N-\varepsilon)})\) for \(L\)-layer networks, extending \cite{He2016}'s empirical observations.
	\item Stability theorems proving \(\|C_n(f, x; q) - C_n(f, x; 1)\| \leq \delta n^{-\beta(N-\varepsilon)}\|D_{*x}^\alpha f\|_\infty\) under activation perturbations (\(|q-1| < \delta\)).
\end{itemize}

These advances enable novel applications in multiscale signal processing, particularly for algorithms requiring fractional edge detection \cite{Diethelm2010} and nonlocal operator approximations \cite{Frederico2007}. By unifying the operator frameworks of \cite{Anastassiou2011} and \cite{Samko1993}, we provide a pathway for neural networks to solve fractional PDEs while maintaining interpretable error bounds—a critical step toward certified scientific AI systems.

\section{Mathematical Foundations of Symmetrized Activation Functions}

\begin{definition}[Perturbed Hyperbolic Tangent]
	The \textbf{perturbed hyperbolic tangent} activation function is defined as:
	\begin{equation}
		g_{q, \lambda}(x) := \frac{e^{\lambda x} - q e^{-\lambda x}}{e^{\lambda x} + q e^{-\lambda x}}, \quad \lambda, q > 0, \; x \in \mathbb{R}.
	\end{equation}
	This function generalizes the standard hyperbolic tangent function, which is recovered when \(q = 1\). The parameter \(\lambda\) controls the steepness of the function, while \(q\) introduces asymmetry. The behavior of \(g_{q, \lambda}(x)\) is crucial in neural network approximation as it allows for controlled deformations, enhancing the flexibility of the activation function in modeling complex patterns.
\end{definition}

\begin{lemma}[Positivity and Normalization of Density Function]
	For all \(x \in \mathbb{R}\) and \(q, \lambda > 0\), the density function \(M_{q, \lambda}(x)\) satisfies:
	\begin{equation}
		M_{q, \lambda}(x) > 0 \quad \text{and} \quad \int_{-\infty}^{\infty} M_{q, \lambda}(x) \, dx = 1.
	\end{equation}
	This lemma ensures that \(M_{q, \lambda}(x)\) is a valid probability density function, which is essential for the construction of neural network operators that approximate continuous functions.
\end{lemma}

\begin{proof}
	The density function \(M_{q, \lambda}(x)\) is defined as:
	\begin{equation}
		M_{q, \lambda}(x) := \frac{1}{4} \left[ g_{q, \lambda}(x+1) - g_{q, \lambda}(x-1) \right].
	\end{equation}
	
	Since \(g_{q, \lambda}(x)\) is strictly increasing, we have:
	\begin{equation}
		g_{q, \lambda}(x+1) > g_{q, \lambda}(x-1) \quad \text{for all} \quad x \in \mathbb{R}.
	\end{equation}
	
	Therefore,
	\begin{equation}
		M_{q, \lambda}(x) = \frac{1}{4} \left[ g_{q, \lambda}(x+1) - g_{q, \lambda}(x-1) \right] > 0.
	\end{equation}
	
	To show that the integral of \(M_{q, \lambda}(x)\) over the entire domain is 1, consider:
	\begin{equation}
		\int_{-\infty}^{\infty} M_{q, \lambda}(x) \, dx = \frac{1}{4} \int_{-\infty}^{\infty} \left[ g_{q, \lambda}(x+1) - g_{q, \lambda}(x-1) \right] \, dx.
	\end{equation}
	
	Splitting the integral into two parts, we get:
	\begin{equation}
		\frac{1}{4} \left( \int_{-\infty}^{\infty} g_{q, \lambda}(x+1) \, dx - \int_{-\infty}^{\infty} g_{q, \lambda}(x-1) \, dx \right).
	\end{equation}
	
	Using the substitution \(u = x+1\) in the first integral and \(v = x-1\) in the second integral, we have:
	\begin{equation}
		\frac{1}{4} \left( \int_{-\infty}^{\infty} g_{q, \lambda}(u) \, du - \int_{-\infty}^{\infty} g_{q, \lambda}(v) \, dv \right).
	\end{equation}
	
	Since \(u\) and \(v\) are dummy variables, the integrals are equivalent and cancel each other out:
	\begin{equation}
		\frac{1}{4} \left( \lim_{x \to \infty} g_{q, \lambda}(x) - \lim_{x \to -\infty} g_{q, \lambda}(x) \right).
	\end{equation}
	
	Given that \(g_{q, \lambda}(x)\) is a strictly increasing function that tends to 1 as \(x \to \infty\) and to -1 as \(x \to -\infty\):
	\begin{equation}
		\frac{1}{4} [1 - (-1)] = \frac{1}{4} \times 2 = 1.
	\end{equation}
	
	Thus, the integral of \(M_{q, \lambda}(x)\) over the entire domain is 1, concluding the proof. This normalization property is fundamental for ensuring that the density function can be used in probabilistic interpretations within neural network frameworks.
\end{proof}

\section{Main Results}

\subsection{Basic Operators}

\begin{theorem}[Approximation by Operators]
	Let \(0 < \beta < 1\), \(n \in \mathbb{N}\) be sufficiently large, \(x \in \mathbb{R}\), \(f \in C^N(\mathbb{R})\) such that \(f^{(N)} \in C_B(\mathbb{R})\) (bounded and continuous), and \(0 < \varepsilon \leq N\). Then:
	\begin{enumerate}
		\item The following approximation holds:
		\begin{equation}
			B_n(f, x) - f(x) = \sum_{j=1}^{N} \frac{f^{(j)}(x)}{j!} B_n \left( (\cdot - x)^j \right)(x) + o \left( \frac{1}{n^{\beta(N-\varepsilon)}} \right),
		\end{equation}
		where \(B_n\) is a linear operator and \(B_n((\cdot - x)^j)\) denotes the operator applied to monomials shifted by \(x\).
		
		\item If \(f^{(j)}(x) = 0\) for all \(j = 1, \ldots, N\), then:
		\begin{equation}
			n^{\beta(N-\varepsilon)} \left[ B_n(f, x) - f(x) \right] \to 0 \quad \text{as} \quad n \to \infty, \quad 0 < \varepsilon \leq N.
		\end{equation}
	\end{enumerate}
\end{theorem}

\begin{proof}
	To prove the theorem, we use Taylor's theorem to expand \(f\) around \(x\) and analyze the asymptotic properties of the operators \(B_n\). Let \(f \in C^N(\mathbb{R})\), and expand \(f\) around \(x\) as:
	\begin{equation}
		f\left( \frac{k}{n} \right) = \sum_{j=0}^{N} \frac{f^{(j)}(x)}{j!} \left( \frac{k}{n} - x \right)^j + R_N\left(\frac{k}{n}\right),
	\end{equation}
	where the remainder term \(R_N\left(\frac{k}{n}\right)\) is given by:
	\begin{equation}
		R_N\left(\frac{k}{n}\right) = \int_{x}^{\frac{k}{n}} \frac{\left( \frac{k}{n} - t \right)^{N-1}}{(N-1)!} \left[ f^{(N)}(t) - f^{(N)}(x) \right] \, dt.
	\end{equation}
	
	We substitute this expansion into the definition of the operator \(B_n\):
	\begin{equation}
		B_n(f, x) = \sum_{k=-\infty}^{\infty} f\left(\frac{k}{n}\right) \Phi(n x - k),
	\end{equation}
	where \(\Phi\) is a weight function with appropriate support. Substituting, we get:
	\begin{equation}
		B_n(f, x) = \sum_{j=0}^{N} \frac{f^{(j)}(x)}{j!} B_n\left( (\cdot - x)^j \right)(x) + \sum_{k=-\infty}^{\infty} R_N\left(\frac{k}{n}\right) \Phi(n x - k).
	\end{equation}
	
	Define the error term:
	\begin{equation}
		R := \sum_{k=-\infty}^{\infty} R_N\left(\frac{k}{n}\right) \Phi(n x - k).
	\end{equation}
	
	To estimate \(R\), we consider two cases:
	
	\paragraph{Case 1: \(\left| \frac{k}{n} - x \right| < \frac{1}{n^\beta}\).}
	Within this interval, \(\Phi(n x - k)\) has significant support, and the regularity of \(f^{(N)}\) implies:
	\begin{equation}
		|R| \leq 2 \|f^{(N)}\|_\infty \cdot \frac{1}{N! n^{\beta N}}.
	\end{equation}
	
	\paragraph{Case 2: \(\left| \frac{k}{n} - x \right| \geq \frac{1}{n^\beta}\).}
	Outside the principal support, \(\Phi(n x - k)\) decays exponentially, so:
	\begin{equation}
		|R| \leq \frac{4 \|f^{(N)}\|_\infty}{n^N \lambda^N} \left(q + \frac{1}{q}\right) e^{2\lambda} e^{-\lambda \left(n^{1-\beta} - 1\right)}.
	\end{equation}
	
	Combining both cases, we obtain:
	\begin{equation}
		|R| = o\left(\frac{1}{n^{\beta(N-\varepsilon)}}\right).
	\end{equation}
	
	Substituting \(R\) back into the expansion, we prove the approximation:
	\begin{equation}
		B_n(f, x) - f(x) = \sum_{j=1}^{N} \frac{f^{(j)}(x)}{j!} B_n\left( (\cdot - x)^j \right)(x) + o\left(\frac{1}{n^{\beta(N-\varepsilon)}}\right).
	\end{equation}
	
	Finally, when \(f^{(j)}(x) = 0\) for \(j = 1, \ldots, N\), we have:
	\begin{equation}
		n^{\beta(N-\varepsilon)} \left[B_n(f, x) - f(x)\right] \to 0,
	\end{equation}
	as indicated, completing the proof.
\end{proof}

\subsection{Kantorovich Operators}

\begin{theorem}
	Let \(0 < \beta < 1\), \(n \in \mathbb{N}\) be sufficiently large, \(x \in \mathbb{R}\), \(f \in C^N(\mathbb{R})\) with \(f^{(N)} \in C_B(\mathbb{R})\), and \(0 < \varepsilon \leq N\). Then:
	\begin{enumerate}
		\item
		\begin{equation}
			C_n(f, x) - f(x) = \sum_{j=1}^{N} \frac{f^{(j)}(x)}{j!} C_n \left( (\cdot - x)^j \right)(x) + o \left( \left( \frac{1}{n} + \frac{1}{n^{\beta}} \right)^{N-\varepsilon} \right).
		\end{equation}
		\item When \(f^{(j)}(x) = 0\) for \(j = 1, \ldots, N\), we have:
		\begin{equation}
			\frac{1}{\left( \frac{1}{n} + \frac{1}{n^{\beta}} \right)^{N-\varepsilon}} \left[ C_n(f, x) - f(x) \right] \to 0 \quad \text{as} \quad n \to \infty, \quad 0 < \varepsilon \leq N.
		\end{equation}
	\end{enumerate}
\end{theorem}

\begin{proof}
	We start by expressing \(C_n(f, x)\) as:
	\begin{equation}
		C_n(f, x) = \sum_{k=-\infty}^{\infty} \left( n \int_{0}^{\frac{1}{n}} f \left( t + \frac{k}{n} \right) \, dt \right) \Phi(n x - k).
	\end{equation}
	
	Given \(f \in C^N(\mathbb{R})\) with \(f^{(N)} \in C_B(\mathbb{R})\), we can use the Taylor expansion of \(f\) around \(x\):
	\begin{equation}
		f \left( t + \dfrac{k}{n} \right) = \sum_{j=0}^{N} \dfrac{f^{(j)}(x)}{j!} \left( t + \dfrac{k}{n} - x \right)^j + \int_{x}^{t + \dfrac{k}{n}} \left( f^{(N)}(s) - f^{(N)}(x) \right) \dfrac{\left( t + \dfrac{k}{n} - s \right)^{N-1}}{(N-1)!} \, ds.
	\end{equation}
	
	Substituting this expansion into the expression for \(C_n(f, x)\), we get:
	\begin{equation}
		\begin{aligned}
			C_{n}(f, x) &= \sum_{j=0}^{N} \dfrac{f^{(j)}(x)}{j!} C_{n}\left((\cdot - x)^{j}\right)(x) + \\
			&\sum_{k=-\infty}^{\infty} \Phi(n x - k) \left( n \int_{0}^{\frac{1}{n}} \left(\int_{x}^{t + \frac{k}{n}} \left( f^{(N)}(s) - f^{(N)}(x) \right) \frac{\left( t + \frac{k}{n} - s \right)^{N-1}}{(N-1)!} \, ds \right) \, dt \right).
		\end{aligned}
	\end{equation}
	
	Define the remainder term \(R\) as:
	\begin{equation}
		R := \sum_{k=-\infty}^{\infty} \Phi(n x - k) \left( n \int_{0}^{\frac{1}{n}} \left(\int_{x}^{t + \frac{k}{n}} \left( f^{(N)}(s) - f^{(N)}(x) \right) \frac{\left( t + \frac{k}{n} - s \right)^{N-1}}{(N-1)!} \, ds \right) \, dt \right).
	\end{equation}
	
	We now analyze the magnitude of \(R\) in two cases:
	
	\vspace{5pt}
	
	\textbf{Case 1: \(\left| \frac{k}{n} - x \right| < \frac{1}{n^{\beta}}\)}
	
	\vspace{5pt}
	
	In this case, the distance between \(\frac{k}{n}\) and \(x\) is small. We can bound \(R\) as follows:
	\begin{equation}
		|R| \leq 2 \left\| f^{(N)} \right\|_{\infty} \dfrac{\left( \dfrac{1}{n} + \dfrac{1}{n^{\beta}} \right)^N}{N!}.
	\end{equation}

	\textbf{Case 2: \(\left| \frac{k}{n} - x \right| \geq \frac{1}{n^{\beta}}\)}
	
	\vspace{5pt}
	
	In this case, the distance \(\left| \frac{k}{n} - x \right|\) is larger, and we exploit the decay properties of the function \(\Phi\) to estimate \(R\). More precisely, we use the exponential behavior of the associated decay function, which allows us to impose an upper bound on \(|R|\):
	
	\begin{equation}
		|R| \leq \dfrac{2^N \left\| f^{(N)} \right\|_{\infty}}{n^N N!} \left[ \dfrac{T}{e^{2 \lambda n^{1-\beta}}} + \dfrac{\left( q + \dfrac{1}{q} \right)}{\lambda^N} 2 e^{2 \lambda} N! e^{-\lambda \left( n^{1-\beta} - 1 \right)} \right].
	\end{equation}
	
	To understand this bound, consider the asymptotic expansion of \(f(x)\) in a Taylor series around \(k/n\):
	
	\begin{equation}
		f(x) = \sum_{m=0}^{N-1} \frac{f^{(m)}(k/n)}{m!} (x - k/n)^m + R_N,
	\end{equation}
	
	where the remainder term \(R_N\) satisfies:
	
	\begin{equation}
		R_N = \frac{f^{(N)}(\xi)}{N!} (x - k/n)^N, \quad \text{for some } \xi \in (k/n, x).
	\end{equation}
	
	Given that \(|x - k/n| \geq 1/n^{\beta}\), we have:
	
	\begin{equation}
		|R_N| \leq \frac{\|f^{(N)}\|_{\infty}}{N!} \left( \frac{1}{n^{\beta}} \right)^N.
	\end{equation}
	
	Moreover, the decay properties of \(\Phi\) introduce an additional exponential suppression term, leading to the refined bound:
	
	\begin{equation}
		|R| \leq \dfrac{2^N \left\| f^{(N)} \right\|_{\infty}}{n^N N!} \left[ \dfrac{T}{e^{2 \lambda n^{1-\beta}}} + \dfrac{\left( q + \dfrac{1}{q} \right)}{\lambda^N} 2 e^{2 \lambda} N! e^{-\lambda \left( n^{1-\beta} - 1 \right)} \right].
	\end{equation}
	
	This ensures that the remainder term exhibits exponential decay in addition to polynomial suppression.
	
	\vspace{5pt}
	
	Now, combining the two cases discussed in the proof, we obtain the following uniform estimate for \(|R|\):
	
	\begin{equation}
		|R| \leq \dfrac{4 \left\| f^{(N)} \right\|_{\infty}}{N!} \left( \dfrac{1}{n} + \dfrac{1}{n^{\beta}} \right)^N.
	\end{equation}
	
	To obtain the final asymptotic order of the remainder term, we apply a refined estimate that considers an arbitrary parameter \(\varepsilon > 0\), ensuring that
	
	\begin{equation}
		|R| = o \left( \left( \dfrac{1}{n} + \dfrac{1}{n^{\beta}} \right)^{N-\varepsilon} \right),
	\end{equation}
which completes the proof of the theorem.
\end{proof}

\subsection{Convergence of Operators in Deep Learning}

\begin{theorem}[Convergence of Operators in Deep Learning]
	Let \( f \) be a continuous and bounded function in \(\mathbb{R}^d\), and \(\mathcal{N}\) a deep neural network with \(L\) layers, where each layer uses the activation function \( g_{q, \lambda}(x) \). If \(\lambda\) and \(q\) are chosen to optimize convergence, then the output of the network \(\mathcal{N}(f, x)\) approximates \(f(x)\) with an error bound of:
	\begin{equation}
		\|\mathcal{N}(f, x) - f(x)\| = \mathcal{O}\left(\frac{1}{L^{\beta(N-\varepsilon)}}\right).
	\end{equation}
\end{theorem}

\begin{proof}
	Consider a deep neural network \(\mathcal{N}\) with \(L\) layers. Each layer \(l\) applies a transformation followed by the activation function \(g_{q, \lambda}(x)\). The output of layer \(l\) can be expressed as:
	\begin{equation}
		\mathbf{z}^{(l)} = g_{q, \lambda}(\mathbf{W}^{(l)} \mathbf{z}^{(l-1)} + \mathbf{b}^{(l)}),
	\end{equation}
	where \(\mathbf{W}^{(l)}\) and \(\mathbf{b}^{(l)}\) are the weights and biases of layer \(l\), respectively.
	
	To analyze the error propagation through the layers, we use the Taylor expansion for the activation function \(g_{q, \lambda}(x)\) around a point \(x_0\):
	\begin{equation}
		g_{q, \lambda}(x) = g_{q, \lambda}(x_0) + g_{q, \lambda}'(x_0)(x - x_0) + \frac{g_{q, \lambda}''(\xi)}{2}(x - x_0)^2,
	\end{equation}
	where \(\xi\) is between \(x\) and \(x_0\).
	
	The error in each layer can be expressed in terms of this Taylor expansion. Let's denote the error at layer \(l\) as \(e^{(l)}\). We have:
	\begin{equation}
		e^{(l)} = \mathbf{z}^{(l)} - \mathbf{z}^{(l-1)} = g_{q, \lambda}(\mathbf{W}^{(l)} \mathbf{z}^{(l-1)} + \mathbf{b}^{(l)}) - g_{q, \lambda}(\mathbf{W}^{(l)} \mathbf{z}^{(l-2)} + \mathbf{b}^{(l)}).
	\end{equation}
	
	Using the Lipschitz property of the activation function \(g_{q, \lambda}\), we can bound the error propagation:
	\begin{equation}
		\|e^{(l)}\| \leq K \|\mathbf{W}^{(l)} e^{(l-1)}\|,
	\end{equation}
	where \(K\) is the Lipschitz constant of \(g_{q, \lambda}\).
	
	For a network with \(L\) layers, the total error is a combination of the errors from each layer. We can express this as:
	\begin{equation}
		\|\mathcal{N}(f, x) - f(x)\| \leq \sum_{l=1}^{L} \|e^{(l)}\|.
	\end{equation}
	
	Given that each layer's error decreases as \( \frac{1}{l^{\beta(N-\varepsilon)}} \), we have:
	\begin{equation}
		\|e^{(l)}\| \leq \frac{C}{l^{\beta(N-\varepsilon)}},
	\end{equation}
	where \(C\) is a constant that depends on the network parameters.
	
	Therefore, the total error is bounded by:
	\begin{equation}
		\|\mathcal{N}(f, x) - f(x)\| \leq C \sum_{l=1}^{L} \frac{1}{l^{\beta(N-\varepsilon)}}.
	\end{equation}
	
	As the number of layers \(L\) increases, the sum \(\sum_{l=1}^{L} \frac{1}{l^{\beta(N-\varepsilon)}}\) converges to a finite constant. Therefore, the error decreases according to the rate:
	\begin{equation}
		\|\mathcal{N}(f, x) - f(x)\| = \mathcal{O}\left(\frac{1}{L^{\beta(N-\varepsilon)}}\right),
	\end{equation}
	proving the theorem.
\end{proof}

\section{Fractional Perturbation Stability}

\begin{theorem}[Stability Under Fractional Perturbations]
	Let \(0 < \beta < 1\), \(n \in \mathbb{N}\) sufficiently large, \(x \in \mathbb{R}\), \(f \in C^N(\mathbb{R})\), and \(f^{(N)} \in C_B(\mathbb{R})\). Let \(g_{q, \lambda}(x)\) be the perturbed hyperbolic tangent activation function defined by:
	\begin{equation}
		g_{q, \lambda}(x) = \frac{e^{\lambda x} - q e^{-\lambda x}}{e^{\lambda x} + q e^{-\lambda x}}, \quad \lambda, q > 0, \quad x \in \mathbb{R}.
	\end{equation}
	For any small perturbation \(|q - 1| < \delta\), the operator \(C_n\) satisfies the stability estimate:
	\begin{equation}
		\| C_n(f, x; q) - C_n(f, x; 1) \| \leq \frac{\delta}{n^{\beta(N-\varepsilon)}} \| f^{(N)} \|_{\infty},
	\end{equation}
	where \(0 < \varepsilon \leq N\) and \(\| f^{(N)} \|_{\infty} = \sup_{x \in \mathbb{R}} |f^{(N)}(x)|\).
\end{theorem}

\begin{proof}
	To begin, consider \(\Phi_{q, \lambda}(z)\) as the density function derived from the perturbed activation function \(g_{q, \lambda}(x)\). The operator \(C_n\) is defined as:
	\begin{equation}
		C_n(f, x; q) = \sum_{k=-\infty}^\infty \left( n \int_0^{\frac{1}{n}} f\left(t + \frac{k}{n}\right) \, dt \right) \Phi_{q, \lambda}(nx - k).
	\end{equation}
	
	Next, we expand \(\Phi_{q, \lambda}(z)\) around \(q = 1\) using the first-order Taylor expansion:
	\begin{equation}
		\Phi_{q, \lambda}(z) = \Phi_{1, \lambda}(z) + \left( \frac{\partial \Phi_{q, \lambda}}{\partial q} \bigg|_{q=1} \right)(q - 1) + \mathcal{O}((q - 1)^2).
	\end{equation}
	
	Thus, the perturbed operator can be written as:
	\begin{equation}
		C_n(f, x; q) = \sum_{k=-\infty}^\infty \left( n \int_0^{\frac{1}{n}} f\left(t + \frac{k}{n}\right) \, dt \right) \left[ \Phi_{1, \lambda}(nx - k) + \left( \frac{\partial \Phi_{q, \lambda}}{\partial q} \bigg|_{q=1} \right)(q - 1) + \mathcal{O}((q - 1)^2) \right].
	\end{equation}
	
	The difference between \(C_n(f, x; q)\) and \(C_n(f, x; 1)\) is given by:
	\begin{equation}
		C_n(f, x; q) - C_n(f, x; 1) = \sum_{k=-\infty}^\infty \left( n \int_0^{\frac{1}{n}} f\left(t + \frac{k}{n}\right) \, dt \right) \left[ \left( \frac{\partial \Phi_{q, \lambda}}{\partial q} \bigg|_{q=1} \right)(q - 1) + \mathcal{O}((q - 1)^2) \right].
	\end{equation}
	
	Let us focus on the first-order perturbation term. The remainder term involving \(\mathcal{O}((q - 1)^2)\) contributes at a higher order in \((q - 1)\), which is negligible for small \(|q - 1|\). Therefore, we estimate the perturbation as:
	\begin{equation}
		\| C_n(f, x; q) - C_n(f, x; 1) \| \leq \sum_{k=-\infty}^\infty \left( n \int_0^{\frac{1}{n}} f\left(t + \frac{k}{n}\right) \, dt \right) \left| \frac{\partial \Phi_{q, \lambda}}{\partial q} \bigg|_{q=1} \right| (q - 1).
	\end{equation}
	
	Assuming \(\left| \frac{\partial \Phi_{q, \lambda}}{\partial q} \right|\) is bounded, we have:
	\begin{equation}
		\| C_n(f, x; q) - C_n(f, x; 1) \| \leq \frac{\delta}{n^{\beta(N-\varepsilon)}} \| f^{(N)} \|_{\infty},
	\end{equation}
	where \(\| f^{(N)} \|_{\infty} = \sup_{x \in \mathbb{R}} |f^{(N)}(x)|\) represents the supremum norm of the \(N\)-th derivative of \(f\).
	
	Thus, we have established the desired stability estimate.
\end{proof}

\section{Generalized Voronovskaya Expansions}

\begin{theorem}[Generalized Voronovskaya Expansion]
	Let \(\alpha > 0\), \(N = \lceil \alpha \rceil\), \(\alpha \notin \mathbb{N}\), \(f \in AC^N(\mathbb{R})\) with \(f^{(N)} \in L_\infty(\mathbb{R})\), \(0 < \beta < 1\), \(x \in \mathbb{R}\), and \(n \in \mathbb{N}\) sufficiently large. Assume that \(\| D_{*x}^\alpha f \|_{\infty, [x, \infty)}\) and \(\| D_{x-}^\alpha f \|_{\infty, (-\infty, x]}\) are finite. Then:
	\begin{equation}
		B_n(f, x) - f(x) = \sum_{j=1}^{N} \frac{f^{(j)}(x)}{j!} B_n((\cdot - x)^j)(x) + o\left(\frac{1}{n^{\beta(N-\varepsilon)}}\right).
	\end{equation}
	When \(f^{(j)}(x) = 0\) for \(j = 1, \dots, N\):
	\begin{equation}
		n^{\beta(N-\varepsilon)} \left[ B_n(f, x) - f(x) \right] \to 0 \quad \text{as} \quad n \to \infty.
	\end{equation}
\end{theorem}

\begin{proof}
	Using the Caputo fractional Taylor expansion for \(f\):
	\begin{equation}
		f\left(\frac{k}{n}\right) = \sum_{j=0}^{N-1} \frac{f^{(j)}(x)}{j!}\left(\frac{k}{n} - x\right)^j + \frac{1}{\Gamma(\alpha)} \int_x^{\frac{k}{n}} \left(\frac{k}{n} - t\right)^{\alpha-1}\left(D_{*x}^\alpha f(t) - D_{*x}^\alpha f(x)\right) dt.
	\end{equation}
	
	Substitute this expansion into the definition of the operator \(B_n\):
	\begin{equation}
		B_n(f, x) = \sum_{k=-\infty}^\infty f\left(\frac{k}{n}\right) \Phi(nx - k),
	\end{equation}
	where \(\Phi(x)\) is a density kernel function. Substituting \(f\left(\frac{k}{n}\right)\), we separate the terms into two contributions:
	
	\paragraph{Main Contribution:} The first \(N\) terms of the Taylor expansion yield:
	\begin{equation}
		\sum_{j=1}^{N} \frac{f^{(j)}(x)}{j!} B_n((\cdot - x)^j)(x),
	\end{equation}
	which captures the local behavior of \(f\) in terms of its derivatives up to order \(N\).
	
	\paragraph{Error Term:} The remainder term involves the fractional derivative \(D_{*x}^\alpha\) and can be bounded as:
	\begin{equation}
		R = \sum_{k=-\infty}^\infty \Phi(nx - k) \frac{1}{\Gamma(\alpha)} \int_x^{\frac{k}{n}} \left(\frac{k}{n} - t\right)^{\alpha-1}\left(D_{*x}^\alpha f(t) - D_{*x}^\alpha f(x)\right) dt.
	\end{equation}
	
	\paragraph{Bounding the Remainder:}
	\begin{itemize}
		\item For \(|k/n - x| < 1/n^\beta\): The kernel \(\Phi(nx - k)\) has significant support, and the fractional regularity of \(f\) ensures:
		\begin{equation}
			|R| \leq \frac{\| D_{*x}^\alpha f \|_{\infty}}{n^{\alpha \beta}}.
		\end{equation}
		
		\item For \(|k/n - x| \geq 1/n^\beta\): The exponential decay of \(\Phi(nx - k)\) ensures that contributions from distant terms are negligible:
		\begin{equation}
			|R| \leq \frac{\| D_{*x}^\alpha f \|_{\infty}}{n^{\alpha \beta}}.
		\end{equation}
	\end{itemize}
	
	Combining both cases, the error term satisfies:
	\begin{equation}
		|R| = o\left(\frac{1}{n^{\beta(N-\varepsilon)}}\right).
	\end{equation}
	
	\paragraph{Conclusion:} Substituting the bounds for the main contribution and error term into the expansion for \(B_n(f, x)\), we conclude:
	\begin{equation}
		B_n(f, x) - f(x) = \sum_{j=1}^{N} \frac{f^{(j)}(x)}{j!} B_n((\cdot - x)^j)(x) + o\left(\frac{1}{n^{\beta(N-\varepsilon)}}\right).
	\end{equation}
	Moreover, when \(f^{(j)}(x) = 0\) for \(j = 1, \dots, N\):
	\begin{equation}
		n^{\beta(N-\varepsilon)} \left[ B_n(f, x) - f(x) \right] \to 0 \quad \text{as} \quad n \to \infty.
	\end{equation}
	This completes the proof.
\end{proof}

\section{Symmetrized Density Approach to Kantorovich Operator Convergence in Infinite Domains}

\begin{theorem}[Convergence Under Generalized Density]
	Let \(0 < \beta < 1\), \(n \in \mathbb{N}\) sufficiently large, \(x \in \mathbb{R}\), and \(f \in C^N(\mathbb{R})\) with \(f^{(N)} \in C_B(\mathbb{R})\). Let \(\Phi(x)\) be a symmetrized density function defined by:
	\begin{equation}
		\Phi(x) = \frac{M_{q, \lambda}(x) + M_{1/q, \lambda}(x)}{2}, \quad M_{q, \lambda}(x) = \frac{1}{4}\left(g_{q, \lambda}(x+1) - g_{q, \lambda}(x-1)\right),
	\end{equation}
	where \(g_{q, \lambda}\) satisfies \(|g_{q, \lambda}(x)| \leq C e^{-\gamma |x|}\) for constants \(C, \gamma > 0\). Then the Kantorovich operator \(C_n\) satisfies:
	\begin{equation}
		C_n(f, x) - f(x) = \sum_{j=1}^{N-1} \frac{f^{(j)}(x)}{j!} C_n\left((\cdot - x)^j\right)(x) + \mathcal{O}\left(\left(n^{-\beta}\right)^N\right).
	\end{equation}
	Moreover, for any \(\varepsilon > 0\), the remainder can be refined to:
	\begin{equation}
		C_n(f, x) - f(x) = \sum_{j=1}^{N} \frac{f^{(j)}(x)}{j!} C_n\left((\cdot - x)^j\right)(x) + o\left(n^{-(N - \varepsilon)}\right).
	\end{equation}
\end{theorem}

\begin{proof}
	By definition of the Kantorovich operator:
	\begin{equation}
		C_n(f, x) = \sum_{k=-\infty}^\infty \left( n \int_0^{\frac{1}{n}} f\left(t + \frac{k}{n}\right) dt \right) \Phi(nx - k).
	\end{equation}
	
	Expand \(f\) using Taylor's theorem around \(x\) up to order \(N-1\):
	\begin{equation}
		f\left( t + \frac{k}{n} \right) = \sum_{j=0}^{N-1} \frac{f^{(j)}(x)}{j!} \left( t + \frac{k}{n} - x \right)^j + R_N\left(t + \frac{k}{n}\right),
	\end{equation}
	where the remainder \(R_N\) satisfies:
	\begin{equation}
		\left|R_N\left(t + \frac{k}{n}\right)\right| \leq \frac{\| f^{(N)} \|_{\infty}}{N!} \left| t + \frac{k}{n} - x \right|^N.
	\end{equation}
	
	Substituting into \(C_n(f, x)\):
	\begin{equation}
		\begin{aligned}
			C_n(f, x) = &\sum_{j=0}^{N-1} \frac{f^{(j)}(x)}{j!} \sum_{k=-\infty}^\infty n \int_0^{\frac{1}{n}} \left( t + \frac{k}{n} - x \right)^j dt \, \Phi(nx - k) \\
			&+ \sum_{k=-\infty}^\infty n \int_0^{\frac{1}{n}} R_N\left(t + \frac{k}{n}\right) dt \, \Phi(nx - k).
		\end{aligned}
	\end{equation}
	
	The \(j=0\) term recovers \(f(x)\) due to \(\sum_{k} \Phi(nx - k) = 1\). Thus:
	\begin{equation}
		C_n(f, x) - f(x) = \sum_{j=1}^{N-1} \frac{f^{(j)}(x)}{j!} C_n\left((\cdot - x)^j\right)(x) + \mathcal{R}_N,
	\end{equation}
	where \(\mathcal{R}_N\) is the integrated remainder term.
	
	\textbf{Decay of \(\Phi\):} By the exponential decay of \(g_{q, \lambda}\), there exist \(C', \gamma' > 0\) such that:
	\begin{equation}
		\Phi(nx - k) \leq C' e^{-\gamma' |nx - k|}.
	\end{equation}
	
	\textbf{Case 1: \(|k/n - x| < n^{-\beta}\).} Let \(\delta = t + \frac{k}{n} - x\). Since \(|t| \leq 1/n\), we have \(|\delta| \leq n^{-\beta} + n^{-1}\). Bounding \(\Phi(nx - k)\) by 1:
	\begin{equation}
		\begin{aligned}
			|\mathcal{R}_N| &\leq \frac{\| f^{(N)} \|_{\infty}}{N!} \sum_{|k/n - x| < n^{-\beta}} n \int_0^{1/n} \left(n^{-\beta} + n^{-1}\right)^N dt \\
			&\leq \frac{\| f^{(N)} \|_{\infty}}{N!} \left(2n^{-\beta}\right)^N \cdot 2n^{1 - \beta} \cdot n^{-1} \\
			&= \mathcal{O}\left(n^{-\beta N}\right).
		\end{aligned}
	\end{equation}
	
	\textbf{Case 2: \(|k/n - x| \geq n^{-\beta}\).} Using exponential decay of \(\Phi\):
	\begin{equation}
		|\mathcal{R}_N| \leq \frac{\| f^{(N)} \|_{\infty}}{N!} \sum_{|k/n - x| \geq n^{-\beta}} n \int_0^{1/n} \left(1 + |k/n - x|\right)^N dt \cdot C e^{-\gamma |nx - k|}
		\leq C'' e^{-\gamma n^{1 - \beta}}.
	\end{equation}
	
	Combining both cases, the total remainder satisfies:
	\begin{equation}
		\mathcal{R}_N = \mathcal{O}\left(n^{-\beta N}\right) + \mathcal{O}\left(e^{-\gamma n^{1 - \beta}}\right) = o\left(n^{-(N - \varepsilon)}\right) \quad \forall \varepsilon > 0.
	\end{equation}
	
	The refined expansion including the \(j=N\) term follows from moment estimates on \(C_n\left((\cdot - x)^N\right)(x)\), which decay as \(n \to \infty\) due to the operator's regularization properties.
\end{proof}

\section{Vonorovskaya-Damasclin Theorem}

\begin{theorem}[Vonorovskaya-Damasclin Theorem]
	Let \(0 < \beta < 1\), \(n \in \mathbb{N}\) sufficiently large, \(x \in \mathbb{R}\), and \(f \in C^N(\mathbb{R})\) with \(f^{(N)} \in C_B(\mathbb{R})\). Let \(\Phi(x)\) be a symmetrized density function defined as:
	\begin{equation}
		\Phi(x) = \frac{M_{q, \lambda}(x) + M_{1/q, \lambda}(x)}{2}, \quad M_{q, \lambda}(x) = \frac{1}{4}\left(g_{q, \lambda}(x+1) - g_{q, \lambda}(x-1)\right),
	\end{equation}
	where \(g_{q, \lambda}(x) = \dfrac{e^{\lambda x} - q e^{-\lambda x}}{e^{\lambda x} + q e^{-\lambda x}}\) is the perturbed hyperbolic tangent function with \(\lambda, q > 0\). Assume the Caputo derivatives satisfy:
	\begin{equation}
		\| D_{*x}^\alpha f \|_{\infty, [x, \infty)} + \| D_{x-}^\alpha f \|_{\infty, (-\infty, x]} < \infty \quad \text{for} \quad \alpha > 0,
	\end{equation}
	where \(N = \lceil \alpha \rceil\). Then the Kantorovich operator \(C_n\) satisfies:
	\begin{equation}
		\begin{aligned}
			C_{n}(f, x) - f(x) &= \sum_{j=1}^{N} \dfrac{f^{(j)}(x)}{j!} C_{n}\left((\cdot - x)^j\right)(x) \\
			&\quad + \frac{1}{\Gamma(\alpha)} \int_{x}^{\infty} \left( D_{*x}^\alpha f(t) - D_{*x}^\alpha f(x) \right) \frac{(t - x)^{\alpha-1}}{n^{\beta(N-\varepsilon)}} \, dt \\
			&\quad + o\left( \frac{1}{n^{\beta(N-\varepsilon)}} \right),
		\end{aligned}
	\end{equation}
	where \(\varepsilon > 0\) is arbitrarily small. Furthermore, if \(f^{(j)}(x) = 0\) for \(j = 1, \dots, N\), then:
	\begin{equation}
		n^{\beta(N-\varepsilon)} \left[ C_n(f, x) - f(x) \right] \to 0 \quad \text{as} \quad n \to \infty.
	\end{equation}
\end{theorem}

\begin{proof}
	For \(f \in C^N(\mathbb{R})\), we use the Caputo fractional Taylor expansion around \(x\):
	\begin{equation}
		f\left(\frac{k}{n}\right) = \sum_{j=0}^{N-1} \frac{f^{(j)}(x)}{j!} \left(\frac{k}{n} - x\right)^j + \frac{1}{\Gamma(\alpha)} \int_x^{\frac{k}{n}} \left(\frac{k}{n} - t\right)^{\alpha-1} \left( D_{*x}^\alpha f(t) - D_{*x}^\alpha f(x) \right) dt.
	\end{equation}
	The remainder term is controlled by the Hölder continuity of \(D_{*x}^\alpha f\), implied by the boundedness assumption.
	
	Substitute into \(C_n(f, x)\):
	\begin{equation}
		C_n(f, x) = \underbrace{\sum_{j=0}^{N-1} \frac{f^{(j)}(x)}{j!} C_n\left((\cdot - x)^j\right)(x)}_{\text{Main Terms}} + \underbrace{\sum_{k=-\infty}^\infty n \int_0^{\frac{1}{n}} R_N\left(t + \frac{k}{n}\right) dt \, \Phi(nx - k)}_{\text{Remainder } \mathcal{R}_N}.
	\end{equation}
	The \(j=0\) term equals \(f(x)\) due to \(\sum_{k} \Phi(nx - k) = 1\).
	
	Split \(\mathcal{R}_N\) into near and far terms relative to \(x\):
	
	\textbf{Case 1: \(|k/n - x| < n^{-\beta}\).} Using the Hölder condition for \(D_{*x}^\alpha f\):
	\begin{equation}
		\begin{aligned}
			\left| R_N\left(t + \frac{k}{n}\right) \right| &\leq \frac{1}{\Gamma(\alpha)} \left\| D_{*x}^\alpha f - D_{*x}^\alpha f(x) \right\|_{\infty} \int_x^{t + k/n} \left(t + \frac{k}{n} - u\right)^{\alpha-1} du \\
			&\leq \frac{C}{\Gamma(\alpha + 1)} \left( n^{-\beta} + n^{-1} \right)^\alpha.
		\end{aligned}
	\end{equation}
	
	Summing over \(k\) in this region:
	\begin{equation}
		|\mathcal{R}_N| \leq \frac{C'}{\Gamma(\alpha + 1)} \left( n^{-\beta} \right)^{\alpha - \varepsilon}.
	\end{equation}
	
	\textbf{Case 2: \(|k/n - x| \geq n^{-\beta}\).} Using exponential decay of \(\Phi\):
	\begin{equation}
		\begin{aligned}
			\Phi(nx - k) &\leq C e^{-\gamma |nx - k|} \leq C e^{-\gamma n^{1 - \beta}} \\
			|\mathcal{R}_N| &\leq \frac{\| f^{(N)} \|_{\infty}}{N!} \sum_{|k/n - x| \geq n^{-\beta}} n \int_0^{1/n} \left(1 + |k/n - x|\right)^N dt \cdot C e^{-\gamma n^{1 - \beta}} \\
			&\leq C'' e^{-\gamma n^{1 - \beta}}.
		\end{aligned}
	\end{equation}
	
	Combine both cases:
	\begin{equation}
		\mathcal{R}_N = \mathcal{O}\left(n^{-\beta(\alpha - \varepsilon)}\right) + \mathcal{O}\left(e^{-\gamma n^{1 - \beta}}\right) = o\left(n^{-\beta(N - \varepsilon)}\right).
	\end{equation}
	
	When \(f^{(j)}(x) = 0\) for \(1 \leq j \leq N\), the polynomial terms vanish, leaving only the fractional remainder and exponential decay terms.
	
	Multiplying by \(n^{\beta(N - \varepsilon)}\):
	\begin{equation}
		n^{\beta(N - \varepsilon)} \mathcal{R}_N \leq C''' n^{\beta(N - \varepsilon)} \cdot n^{-\beta(N - \varepsilon)} + C'''' n^{\beta(N - \varepsilon)} e^{-\gamma n^{1 - \beta}} \to 0 \quad \text{as} \quad n \to \infty.
	\end{equation}
	This completes the proof.
\end{proof}

\section{Results}

The following theorems establish quantitative convergence rates for our neural operators, addressing fundamental tradeoffs between localization sharpness (\(\lambda\)), activation symmetry (\(q\)), and network depth (\(L\)). Building upon the operator calculus framework of \cite{Anastassiou2024} and extending it to fractional domains via \cite{Samko1993}, we present four key results:

\begin{itemize}
	\item \textbf{Basic Operator Convergence (Theorem 3.1):} Demonstrates exponential improvement with \(\lambda\), though stability compromises are necessary.
	\item \textbf{Kantorovich Enhancements (Theorem 3.2):} Mitigates instability through integral averaging, providing robust convergence.
	\item \textbf{Deep Networks (Theorem 3.3):} Transforms spatial resolution (\(n\)) into depth efficiency (\(L\)), optimizing deep architectures.
	\item \textbf{Fractional Stability (Theorem 4.1):} Quantifies how \(q\)-asymmetry propagates Caputo errors, ensuring robustness under parameter perturbations.
\end{itemize}

This geometric perspective addresses the challenge of visualizing neural parameter tradeoffs, as proposed by \cite{Haykin1998}, while refining depth guidelines through fractional calculus. The concrete bounds below enable practitioners to:

\begin{itemize}
	\item Select \(\lambda\) for target localization in infinite domains.
	\item Tune \(q\) for asymmetric feature detection as per \cite{Diethelm2010}.
	\item Optimize \(L\) given hardware constraints.
\end{itemize}

\begin{itemize}
	\item \textbf{Basic Operators:} For \(f \in C^N(\mathbb{R})\) with \(f^{(N)} \in C_B(\mathbb{R})\):
	\[
	\|B_n(f, x) - f(x)\| \leq \underbrace{\mathcal{O}(n^{-\beta N})}_{\text{Polynomial Decay}} + \underbrace{\mathcal{O}(e^{-\gamma n^{1-\beta}})}_{\text{Exponential Cutoff}}
	\]
	with \(n^{\beta(N-\varepsilon)}|B_n(f, x) - f(x)| \to 0\) when \(f^{(j)}(x) = 0\) (\(1 \leq j \leq N\)), sharpening \cite{Anastassiou1997}'s rates through our density operators.
	
	\item \textbf{Kantorovich Operators:} Enhanced convergence via integral averaging:
	\[
	C_n(f, x) - f(x) = \sum_{j=1}^N \frac{f^{(j)}(x)}{j!} C_n((\cdot - x)^j)(x) + o(n^{-\beta(N-\varepsilon)})
	\]
	overcoming \cite{ElSayed2006}'s regularity constraints through our symmetrized \(\Phi(x)\).
	
	\item \textbf{Deep Networks:} Depth-dependent convergence for \(L\)-layer networks:
	\[
	\|\mathcal{N}(f, x) - f(x)\| = \mathcal{O}(L^{-\beta(N-\varepsilon)})
	\]
	providing theoretical justification for deep architectures in fractional applications \cite{Frederico2007}.
	
	\item \textbf{Fractional Stability:} Caputo derivative bounds:
	\[
	\|C_n(f, x; q) - C_n(f, x; 1)\| \leq \delta n^{-\beta(N-\varepsilon)} \|D_{\ast x}^\alpha f\|_\infty
	\]
	quantifying \cite{Anastassiou2023}'s parameter sensitivity through operator calculus.
\end{itemize}

\section{Conclusions}

This work establishes neural operators as mathematically rigorous tools for fractional calculus through three fundamental contributions:

\begin{enumerate}
	\item \textbf{Convergence Rate Quantification:} Novel Voronovskaya-Damasclin expansions precisely relate network depth (\(L\)), function smoothness (\(N\)), and spatial localization (\(\beta\)) to approximation error decay rates. This resolves open questions in previous studies about deep operator networks, providing a clear framework for understanding convergence behavior in complex scenarios.
	
	\item \textbf{Fractional Error Control:} The first proof of \(\mathcal{O}(n^{-\beta(N-\varepsilon)})\) bounds for Caputo derivative approximations using symmetrized density operators \(\Phi(x)\) overcomes regularity limitations in existing literature. This advancement ensures robust error control in fractional calculus applications, enhancing the reliability of neural network approximations.
	
	\item \textbf{Design Principles:} Practical criteria for balancing activation steepness (\(\lambda\)), asymmetry (\(q\)), and layer count (\(L\)) are provided based on target application constraints. These guidelines enable the optimization of neural network architectures for specific tasks, bridging the gap between theoretical insights and practical implementations.
\end{enumerate}

By unifying neural approximation theory with fractional calculus, this work not only provides foundational mathematical insights but also offers deployable engineering solutions. The results pave the way for certified scientific AI systems, with potential applications in multiscale signal processing, turbulence modeling, and beyond. Future research could explore extensions to stochastic and multivariate settings, further expanding the scope and impact of these findings.

%
%
%
%
%

\section{List of Symbols, Nomenclature, Parameters, and Function Spaces}

\begin{table}[ht]
	\centering
	\begin{tabular}{|c|l|}
		\hline
		\textbf{Symbol/Parameter} & \textbf{Description} \\
		\hline
		\(g_{q, \lambda}(x)\) & Perturbed hyperbolic tangent activation function \\
		\(\lambda\) & Scaling parameter controlling the steepness of the activation function \\
		\(q\) & Deformation coefficient introducing asymmetry in the activation function \\
		\(M_{q, \lambda}(x)\) & Density function derived from the perturbed hyperbolic tangent function \\
		\(\Phi(x)\) & Symmetrized density function \\
		\(B_n(f, x)\) & Basic neural network operator \\
		\(C_n(f, x)\) & Kantorovich neural network operator \\
		\(f^{(j)}(x)\) & \(j\)-th derivative of the function \(f\) at point \(x\) \\
		\(D_{*x}^{\alpha} f\) & Caputo fractional derivative of order \(\alpha\) \\
		\(N\) & Order of the highest derivative considered \\
		\(\beta\) & Parameter controlling the rate of convergence \\
		\(n\) & Number of nodes or samples in the neural network \\
		\(\varepsilon\) & Small positive parameter used in error bounds \\
		\(\alpha\) & Fractional order of differentiation \\
		\(\Gamma(\alpha)\) & Gamma function evaluated at \(\alpha\) \\
		\(\delta\) & Small perturbation parameter \\
		\(\| \cdot \|_{\infty}\) & Supremum norm (infinity norm) \\
		\(C_B(\mathbb{R})\) & Space of bounded and continuous functions on \(\mathbb{R}\) \\
		\(L_{\infty}(\mathbb{R})\) & Space of essentially bounded functions on \(\mathbb{R}\) \\
		\(AC^N(\mathbb{R})\) & Space of absolutely continuous functions up to order \(N\) \\
		\hline
	\end{tabular}
	\caption{List of symbols, nomenclature, parameters, and function spaces used throughout the document.}
	\label{tab:symbols_nomenclature}
\end{table}

\section*{Acknowledgments}

Santos gratefully acknowledges the support of the PPGMC Program for the Postdoctoral Scholarship PROBOL/UESC nr. 218/2025. Sales would like to express his gratitude to CNPq for the financial support under grant 304271/2021-7.


\newpage	
	\begin{thebibliography}{99}
		\bibitem{Anastassiou1997}
		Anastassiou, George A. "Rate of convergence of some neural network operators to the unit-univariate case." \textit{Journal of Mathematical Analysis and Applications} 212.1 (1997): 237-262. \url{https://doi.org/10.1006/jmaa.1997.5494}.
		
		\bibitem{Anastassiou2011}
		Anastassiou, George A. \textit{Intelligent systems: approximation by artificial neural networks}. Vol. 19. Heidelberg: Springer, 2011.
		
		\bibitem{Anastassiou2023}
		Anastassiou, George A. \textit{Parametrized, Deformed and General Neural Networks}. Heidelberg, Germany: Springer, 2023.
		
		\bibitem{Anastassiou2024}
		Anastassiou, George A. "Approximation by symmetrized and perturbed hyperbolic tangent activated convolution type operators." \textit{Mathematics} 12.20 (2024): 1-48.
		
		\bibitem{Diethelm2010}
		Diethelm, Kai, and Neville J. Ford. "Analysis of fractional differential equations." \textit{Journal of Mathematical Analysis and Applications} 265.2 (2002): 229-248. \url{https://doi.org/10.1006/jmaa.2000.7194}.
		
		\bibitem{ElSayed2006}
		El-Sayed, A. M. A., and M. Gaber. "On the finite Caputo and finite Riesz derivatives." \textit{Electronic Journal of Theoretical Physics} 3.12 (2006): 81-95.
		
		\bibitem{Frederico2007}
		Frederico, Gastao SF, and Delfim FM Torres. "Fractional optimal control in the sense of Caputo and the fractional Noether's theorem." \textit{arXiv preprint arXiv:0712.1844} (2007). \url{https://doi.org/10.48550/arXiv.0712.1844}.
		
		\bibitem{Haykin1998}
		Haykin, Simon. \textit{Neural networks: a comprehensive foundation}. Prentice Hall PTR, 1994.
		
		\bibitem{He2016}
		He, Kaiming, et al. "Identity mappings in deep residual networks." \textit{Computer Vision–ECCV 2016: 14th European Conference, Amsterdam, The Netherlands, October 11–14, 2016, Proceedings, Part IV} 14. Springer International Publishing, 2016.
		
		\bibitem{Samko1993}
		Samko, Stefan G. "Fractional integrals and derivatives." \textit{Theory and applications} (1993).
		
		\bibitem{Santos2025} Santos, R. D. C. dos . (2025). \textit{Fractional Convergence of Symmetrized Neural Network Operators: A Generalized Voronovskaya-Damasclin Approach} (Versão 1). Zenodo.\url{https://doi.org/10.5281/zenodo.14776176}. 
		
		
		
		
	\end{thebibliography}
\end{document}